\documentclass{article}

\PassOptionsToPackage{round, compress}{natbib}

 \usepackage[preprint]{neurips_2026}

\usepackage[utf8]{inputenc} 
\usepackage[T1]{fontenc}    
\usepackage{hyperref}       
\usepackage{url}            
\usepackage{booktabs}       
\usepackage{amsfonts}       
\usepackage{nicefrac}       
\usepackage{microtype}      

\usepackage{tikz}
\usetikzlibrary{arrows.meta}
\definecolor{grad1}{RGB}{0, 90, 181}
\definecolor{grad2}{RGB}{200, 50, 50}
\definecolor{dotcolor}{RGB}{40, 40, 40}
\usepackage{todonotes}
\usepackage{enumitem}
\usepackage{bbm}
\usepackage{amsmath}
\usepackage{amssymb}
\usepackage{wrapfig}
\usepackage{float}
\usepackage{xcolor}
\definecolor{dred}{RGB}{153,80,43}
\definecolor{dblue}{RGB}{0,114,178}
\hypersetup{
    colorlinks=true,
    breaklinks=true,
    urlcolor=dblue, 
    linkcolor=dred, 
    citecolor=dblue
}
\usepackage{caption}

\usepackage{amsthm}
\usepackage{thmtools}

\declaretheorem[style=plain, name=Open Question]{question}\declaretheorem[style=plain]{corollary}

\declaretheorem[style=plain]{definition}

\declaretheorem[style=definition]{counterexample}

\newtheorem{skalsetheorem}{Theorem}

\makeatletter
\renewcommand{\paragraph}{%
  \@startsection{paragraph}{4}{\z@}%
                {0ex}%
                {-1em}%
                {\normalsize\bf}%
}
\makeatother

\usepackage{proof-at-the-end}
\pgfkeys{/prAtEnd/global custom defaults/.style={restate, no link to proof}}

\usepackage[capitalize,noabbrev,nameinlink]{cleveref}
\creflabelformat{equation}{#2#1#3}
\crefname{appendix}{Appendix}{Appendices}
\Crefname{appendix}{Appendix}{Appendices}
\crefname{skalsetheorem}{Theorem}{Theorems}
\Crefname{skalsetheorem}{Theorem}{Theorems}
\crefname{conjecture}{Conjecture}{Conjectures}
\Crefname{conjecture}{Conjecture}{Conjectures}
\crefname{question}{Open Question}{Open Questions}
\Crefname{question}{Open Question}{Open Questions}
\crefname{counterexample}{Counterexample}{Counterexamples}
\Crefname{counterexample}{Counterexample}{Counterexamples}

\newcommand{\E}{\mathbb{E}}
\newcommand{\cS}{\mathcal{S}}
\newcommand{\cA}{\mathcal{A}}
\newcommand{\cR}{\mathcal{R}}
\newcommand{\cT}{\mathcal{T}}
\newcommand{\cM}{\mathcal{M}}
\newcommand{\cF}{\mathcal{F}}

\newcommand{\nJ}{\nabla J}

\newcommand{\dPi}{\Pi} 

\newcommand{\nsPi}{{\Pi_\text{NS}}}

\title{Imperfect World Models are Exploitable}

\author{%
  \textbf{Logan Mondal Bhamidipaty}\textsuperscript{1}\thanks{\texttt{l.m.bhamidipaty@sms.ed.ac.uk}} \quad
  \textbf{Esmeralda S.\ Whitammer}\textsuperscript{1} \quad
  \textbf{David Abel}\textsuperscript{1} \\[1.5mm]
  \textbf{Mykel J.\ Kochenderfer}\textsuperscript{2} \quad
  \textbf{Subramanian Ramamoorthy}\textsuperscript{1} \\[3mm]
  \textsuperscript{1}University of Edinburgh \qquad
  \textsuperscript{2}Stanford University
}

\begin{document}

\maketitle

\begin{abstract}
We propose a novel definition of model exploitation in reinforcement learning. Informally, a world model is exploitable if it implies that one policy should be strictly preferred over another while the environment’s true transition model implies the reverse. We analogize our definition with a prior characterization of reward hacking but show that the associated proof of inevitability does not transfer to exploitation. To overcome this obstruction, we develop a general theory of reward hacking and model exploitation that proves that exploitation is essentially unavoidable on large policy sets and yields the corresponding claim for hacking as a special case. Unfortunately, we also find that the conditions that guarantee unhackability in finite policy sets have no counterpart that precludes exploitation. Consequently, we introduce a relaxed notion of exploitation and derive a safe horizon within which it can be avoided. Taken together, our results establish a formal bridge between reward hacking and model exploitation and elucidate the limits of safe planning in world models.
\end{abstract}

\section{Introduction}

Efficient real-world planning typically requires approximating complex dynamics \citep{simon1955behavioral, javed2024big}. In the best case, well-chosen approximations unravel otherwise intractable problems, as with admissible heuristics in search \citep{hart1968formal}, model reduction in dynamical systems \citep{antoulas2005approximation}, and latent-space world models in sequential decision-making \citep{hafner2019learning}. In the worst case, however, ill-chosen approximations can compromise safety, as with mimicry in evolutionary biology \citep{bates1862mimicry}, arbitrage in financial markets \citep{ross1976arbitrage}, and speedrunning in video games \citep{scully2014practiced}. In reinforcement learning (RL), this liability arises not only exogenously from adversaries but also endogenously from the optimizer itself: an agent trained to maximize total expected reward under an imperfect model\footnote{We use ``model'' to mean the transition model of a Markov decision process (\Cref{sec:prelims}) or an agent's approximation to it, sometimes called a world model \citep{ha2018worldmodels}. We do not use it to mean the policy (as a generative model of actions) as is sometimes meant in other RL literature.} may discover behavior that performs well in simulation but poorly in reality \citep{jakobi1995noise, tobin2017domain}. This failure mode is called \textbf{model exploitation} \citep{ha2018worldmodels, janner2019trust}, and a natural question is whether it can be avoided.


This question pertains to notable results in RL theory including the simulation lemma \citep{kearns2002near} and the value equivalence principle \citep{grimm2020value}. Both implicitly use predictive performance to measure the quality of a model's approximation. The simulation lemma does so cardinally by bounding the error from optimizing under an imperfect model, while the value equivalence principle does so locally by characterizing when two models induce identical Bellman updates. Modern world models also emphasize predictive performance, typically learning latent-space models of transitions \citep{hafner2019learning, balestriero2025lejepa}.

We argue that while this emphasis on performance is necessary for efficient planning, it is not sufficient for safe planning. Recent work on reward specification \citep{hadfield2017inverse, abel2021expressivity} views reward functions as observations of the designer's intent to communicate a goal, which may be an imperfect reflection of a true objective. Empirical transition models are similarly imperfect. In this setting, the natural notion of safety is ordinal rather than cardinal (\textit{does the model preserve which policies are better than which?}) and global rather than local (\textit{does it do so across the entire policy set, not just on the chosen classes?}). Concretely, a proxy model is \emph{safe} in our sense if it does not invert the policy ordering induced by the true model. That is, whenever reality prefers a policy $\pi$ over $\pi'$, a safe approximation should not prefer $\pi'$ over $\pi$. When such an inversion exists, we say the pair of transition models is \emph{exploitable} (\Cref{def:exploitability}).

Our formalization of model exploitation is analogous to the notion of reward hacking in \citet{skalse2022defining}. Unfortunately, their characterization of when hacking exists and is avoidable does not transfer to exploitation (\cref{sec:Distinguishing}). We overcome this by introducing a more general theory, from which four main contributions emerge.
\begin{enumerate}[leftmargin=*, itemsep=2pt, topsep=4pt]
    \item We propose a new definition of model exploitation (\cref{def:exploitability}) that, to the best of our knowledge, is the first ordinal notion of safety for world models.
    \item We develop a unified theory of reward hacking and model exploitation that characterizes exploitation on common policy sets (\cref{cor:exploitability,cor:subsets}), yields \hyperref[thm:skalse-thm1]{Theorem~S1} of \citet{skalse2022defining} as a special case (\cref{cor:skalse}), and provides a shared geometric intuition for both phenomena (\cref{lem:exploitation,lem:equivalence}).
    \item We give a construction reducing any instance of model exploitation to one of reward hacking (\cref{prop:exploitation-implies-hacking}), but show that a converse construction does not exist (\cref{ce:hacking-exploition}). We further prove that the conditions that guarantee unhackability are insufficient to guarantee unexploitability (\cref{ce:finite-exploitation}).
    \item We introduce a relaxed notion of model exploitation, $\varepsilon$-exploitability (\cref{def:eps-exploitability}), and use the tight form of the simulation lemma \citep{kearns2002near} recently proved by \cite{lobeloptimal} to derive a safe horizon for discounted Markov decision processes within which $\varepsilon$-unexploitability is guaranteed (\cref{thm:safe-horizon}).
\end{enumerate}

\section{Preliminaries}
\label{sec:prelims}

We first recap the preliminaries for reinforcement learning \citep{sutton1998reinforcement} and introduce our definition of model exploitation. 

\subsection{Reinforcement learning} We consider a \textbf{Markov decision process (MDP)} $\cM = (\cS, \cA, \cT, d_0, \cR, \gamma)$ where $\cS$ is the state space, $\cA$ is the action space, $\cT \colon \cS \times \cA \to \Delta(\cS)$ is the transition model, $d_0 \in \Delta(\cS)$ is the initial state distribution, $\cR \colon \cS \times \cA \to \mathbb{R}$ is the reward function, and $\gamma \in [0, 1)$ is the discount factor. As in \citet{skalse2022defining}, we assume finite $\cS$ with all states reachable and finite $\cA$ with $|\cA| > 1$.

A \textbf{stationary policy} $\pi \colon \cS \to \Delta(\cA)$ maps each state to a distribution over actions. A \textbf{non-stationary policy} $\pi = (\pi_0, \pi_1, \pi_2, \dots)$ is a sequence of mappings $\pi_t \colon \cS \to \Delta(\cA)$, allowing the action distribution to depend on the time step. Rolling out a policy produces a trajectory $\tau = (s_0, a_0, r_0, \dots)$ whose discounted return is $G(\tau) = \sum_{t=0}^\infty \gamma^t r_t$. The \textbf{value} of a policy is its expected return $J(\pi) = \E_{\tau \sim \pi}[G(\tau)]$, or, equivalently, $J(\pi) = \langle \cR, \cF^\pi_\cT \rangle$ where $\cF^\pi_\cT(s, a) = \E_{\tau \sim \pi}[\sum_{t=0}^\infty \gamma^t \mathbbm{1}(s_t = s, a_t = a)]$ are the \textbf{discounted visit counts} of $\pi$ under $\cT$. We make the standard assumption in classical RL and optimal control \citep{sutton1998reinforcement, bertsekas2012dynamic} that designers use value to compare policies and prefer high-value policies over low-value ones.

An \textbf{environment} is an MDP without a reward function $(\cS, \cA, \cT, d_0, \_, \gamma)$,\footnote{This is sometimes called a controlled Markov process.} and a \textbf{task} is an MDP without a transition model $(\cS, \cA, \_, d_0, \cR, \gamma)$. When considering multiple reward functions in a fixed environment, we write $J_\cR(\pi) = \langle \cR, \cF^\pi \rangle$ (dropping the subscript on $\cF^\pi$ when $\cT$ is fixed). When considering multiple transition models in a fixed task, we write $J_\cT(\pi) = \langle \cR, \cF^\pi_\cT \rangle$. We use $J_i$ as shorthand for $J_{\cR_i}$ or $J_{\cT_i}$ depending on whether the index ranges over reward functions or transition models.

Since $\cS$ and $\cA$ are finite, the space of stationary policies can be represented as a product of $|\cS|$ probability simplices $\Delta(\cA)^{|\cS|}$. We write $\Pi^+$ for the interior of this space, consisting of all policies with $\pi(a\mid s) > 0$ for every $s \in \cS$ and $a \in \cA$. Unless otherwise specified, openness is taken with respect to $\Pi^+$. We provide a reference list of notation and terminology in \cref{app:notation_table}.

\subsection{Defining model exploitation and reward hacking}

\begin{figure}[t]
    \vspace*{-1em}
  \centering
  \includegraphics[width=\textwidth]{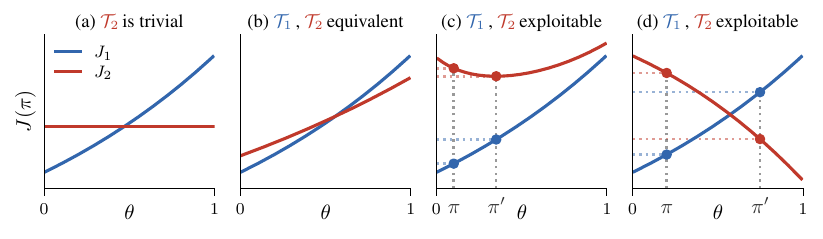}\\[-1em]
  \caption{Taxonomy of transition model relationships in a 3-state MDP
  (\cref{app:figure-mdp}).
  Policies are parameterized by $\pi_\theta(a_0 \mid s) = \theta$ for all~$s \in \cS$. Each panel has the same \textcolor{grad1}{$\cT_1$} but a unique \textcolor{grad2}{$\cT_2$}.
  \textbf{(a)}~\textcolor{grad2}{$\cT_2$} is trivial, so \textcolor{grad2}{$J_2$} is constant.
  \textbf{(b)}~\textcolor{grad1}{$\cT_1$},
  \textcolor{grad2}{$\cT_2$} are equivalent (in policy ordering): both value functions increase
  in~$\theta$, preserving the policy ordering.
  \textbf{(c, d)}~\textcolor{grad1}{$\cT_1$},
  \textcolor{grad2}{$\cT_2$} are exploitable: dotted lines show the
  exploiting pair $(\pi, \pi')$ from \cref{def:exploitability}, with
  \textcolor{grad1}{$J_1(\pi') > J_1(\pi)$} but
  \textcolor{grad2}{$J_2(\pi) > J_2(\pi')$}.}
  \label{fig:taxonomy}
\end{figure}

We define model exploitation as a binary relation between two transition models.

\begin{definition}[Model Exploitation]
\label{def:exploitability}
Transition models $\cT$ and $\cT'$ are \textbf{exploitable} relative to a policy set $\Pi$ and a task $(\cS, \cA, \_, d_0, \cR, \gamma)$ if there exist $\pi, \pi' \in \Pi$ such that
\begin{equation*}
    J_\cT(\pi) > J_\cT(\pi') \quad \text{and} \quad J_{\cT'}(\pi') > J_{\cT'}(\pi),
\end{equation*}
otherwise they are \textbf{unexploitable}.
\end{definition}

Intuitively, model exploitation says that optimizing under one transition model may look like a mistake under the other. This definition directly analogizes the notion of reward hacking in \citet{skalse2022defining} with transition models in place of reward functions. We restate their definition below.

\begin{definition}[Reward Hacking, \citet{skalse2022defining}]
\label{def:hackability}
Reward functions $\cR$ and $\cR'$ are \textbf{hackable} relative to a policy set $\Pi$ and an environment $(\cS, \cA, \cT, d_0, \_, \gamma)$ if there exist $\pi, \pi' \in \Pi$ such that
\begin{equation*}
    J_\cR(\pi) > J_\cR(\pi') \quad \text{and} \quad J_{\cR'}(\pi') > J_{\cR'}(\pi),
\end{equation*}
otherwise they are \textbf{unhackable}.
\end{definition}

Finally, we introduce two properties of value functions that separate interesting cases from degenerate ones. We say two value functions $J_1$ and $J_2$ are \textbf{equivalent} on $\Pi$ if they induce the same ordering on $\Pi$ and that a value function $J$ is \textbf{trivial} on $\Pi$ if $J(\pi)$ is constant on $\Pi$ (\cref{fig:taxonomy}). Informally, equivalence means that two value functions agree on every pairwise comparison between policies, while triviality means there are no meaningful comparisons to make. As shorthand, we also say transition models and reward functions are trivial or equivalent if the value functions they induce are. 

\subsection{Prior characterizations of reward hacking}
\label{sec:prior-hacking}

To make the paper more self-contained, we restate two relevant theorems about reward hacking from \citet{skalse2022defining} and briefly explain their significance.

\begin{skalsetheorem}[Theorem 1 of \citet{skalse2022defining}]
\label{thm:skalse-thm1}
In any environment $(\cS, \cA, \cT, d_0, \_, \gamma)$, if $\Pi$ contains an open set, then any pair of reward functions that are unhackable and non-trivial on $\Pi$ are equivalent on $\Pi$.
\end{skalsetheorem}

\begin{skalsetheorem}[Theorem 2 of \citet{skalse2022defining}, corrected]
\label{thm:skalse-thm2}
For any environment $(\cS, \cA, \cT, d_0, \_, \gamma)$, any finite set of policies $\dPi$ whose visit counts are not collinear, and any reward function $\cR_1$, there is a non-trivial reward function $\cR_2$ such that $\cR_1$ and $\cR_2$ are unhackable but not equivalent.
\end{skalsetheorem}

Informally, \hyperref[thm:skalse-thm1]{Theorem~S1} says that, under mild regularity conditions, no interesting proxy reward (one that is non-trivial, non-equivalent, and unhackable) exists on sufficiently large policy sets. This result would be discouraging were it not for \hyperref[thm:skalse-thm2]{Theorem~S2}, which guarantees that such a proxy always exists on finite policy sets.\footnote{\hyperref[thm:skalse-thm2]{Theorem~S2} includes a small correction to the original statement, which we explain in \cref{sec:finite} and \cref{app:skalse-theorems}.} As we show in \cref{sec:results}, our unified theory of model exploitation and reward hacking recovers an analogous impossibility result to \hyperref[thm:skalse-thm1]{Theorem~S1} but, unfortunately, has no counterpart to the existence guarantee of \hyperref[thm:skalse-thm2]{Theorem~S2}.

 \section{Results}
 \label{sec:results}

 To understand when model exploitation is avoidable, we ask a single question with increasing precision: \textit{on which policy sets can we find non-trivial, non-equivalent, unexploitable transition model pairs?} We begin with the largest possible policy set and progressively tighten the scope.

\subsection{Non-stationary policies}

 Consider first the set of all non-stationary policies $\nsPi$. If exploitation can be avoided here, it can be avoided on any policy set, since all policy sets are subsets of $\nsPi$. Unfortunately, this is false: we find that no interesting unexploitability exists.

\begin{theoremEnd}{proposition}[No unexploitability on $\nsPi$]
\label{prop:non-stationary}
    On the set of all non-stationary policies, every non-trivial, non-equivalent pair of transition models is exploitable.
\end{theoremEnd}
\begin{proofEnd}
    The proof follows the same argument as the analogous result 
    in \citet{skalse2022defining}, with transition models in 
    place of reward functions. The key observation is that 
    $J_i(\pi_\lambda)$ is affine in $\lambda$ by linearity of 
    expectation, regardless of whether $i$ indexes reward 
    functions or transition models.
\end{proofEnd}

We defer all formal proofs to \cref{app:proofs}. The proof uses a similar construction to the one in \citet{skalse2022defining}, relying on the richness of $\nsPi$ to find an exploiting policy pair, but says nothing about the set of stationary policies, where the question is considerably harder.

\subsection{Distinguishing model exploitation from reward hacking}
\label{sec:Distinguishing}

Having seen that interesting unexploitability is impossible in the most general policy set, a natural question is whether it can be avoided in smaller policy sets, such as deterministic or optimal policies. From \hyperref[thm:skalse-thm1]{Theorem~S1} of \citet{skalse2022defining}, we know that reward hacking is inevitable on any policy set that contains an open subset (under mild regularity conditions). Since our proof of \cref{prop:non-stationary} is their argument with transition models in place of reward functions, one might expect the same substitution to give analogues of \hyperref[thm:skalse-thm1]{Theorems~S1} and \hyperref[thm:skalse-thm2]{S2}. Unfortunately, it does not, as we argue in this section.

\subsubsection{Distinguishing transition models and reward functions}
\label{sec:anatomy}

Transition models and reward functions have significantly different geometries that make it hard to reason about them in a unified manner. Two of these differences are important for our results. \textit{First}, the policy value function $J(\pi) = \langle \cR, \cF_\cT^\pi \rangle$ is linear in the reward function but nonlinear in the transition model's probability mass function. This means that the linear-algebraic and real-analytic machinery that \citet{skalse2022defining} develop for their characterization of reward hacking generally does not transfer to the model exploitation setting, so analogous results (if they exist) must rely on different methods. \textit{Second}, the only constraint imposed on the reward function is typically boundedness: rewards may not go to infinity. Transition models, on the other hand, are probability mass functions, which are nonnegative and sum to 1 at each state. This means reward functions live in the full Euclidean space $\mathbb{R}^{|\cS||\cA|}$, which is closed under negation, scaling, and addition, but transition models live on a compact product of simplices $\Delta(\cS)^{|\cS||\cA|}$, which does not have vector space structure (at least in the probability domain).

The primary consequence of these differences is that neither the techniques nor the results from reward hacking cleanly apply to model exploitation. In fact, as we show next, results that are true for hacking are sometimes false for exploitation.

\subsubsection{Finite policy sets}
\label{sec:finite}

Unlike reward hacking, model exploitation persists even in finite policy sets. Formally, \hyperref[thm:skalse-thm2]{Theorem~S2} of \citet{skalse2022defining} shows that, together with some mild regularity conditions, a non-trivial, non-equivalent unhackable reward pair always exists on any finite policy set. 

Making the corresponding existence guarantee for unexploitable transition model pairs requires ruling out two degeneracies that have no analogue for reward hacking. When $\gamma = 0$, the value $J(\pi) = \sum_s d_0(s) \sum_a \pi(a|s)\,\cR(s,a)$ is independent of $\cT$, so every transition model pair is equivalent. The same obstruction applies when both $\cR$ and every policy in $\Pi$ are stateless (i.e., $\cR(s,a) = \cR(a)$ and $\pi(a|s) = \pi(a)$ for all $s \in \cS$), since then $J(\pi) = \frac{1}{1-\gamma}\sum_a \cR(a)\,\pi(a)$. In both cases, all transition models are equivalent.

In examining whether the remaining regularity conditions of \citet{skalse2022defining} suffice once these degeneracies are excluded, we noticed that \hyperref[thm:skalse-thm2]{Theorem~S2} requires an additional hypothesis that is not present in the original statement but implicitly used in the proof. (We view this as a minor oversight that does not affect the broader conclusions of their work, since the condition is degenerate and easily excluded.)

\begin{theoremEnd}{proposition}[{Non-collinearity is necessary for {\expandafter\hyperref[thm:skalse-thm2]{Theorem~S2}} of \citet{skalse2022defining}}]
\label{prop:collinear-visit-counts}
On any policy set with collinear visit counts (i.e., all visit counts $\cF^\pi$ for $\pi \in \Pi$ lie on a single line in $\mathbb{R}^{|\cS||\cA|}$), every non-trivial, non-equivalent pair of reward functions is hackable.
\end{theoremEnd}
\begin{proofEnd}
Collinearity means there exist $\pi_0, \pi_1 \in \Pi$ with $\cF^{\pi_0} \neq \cF^{\pi_1}$ such that $\cF^\pi = \cF^{\pi_0} + c_\pi(\cF^{\pi_1} - \cF^{\pi_0})$ for distinct scalars $c_\pi$. The value of $\pi$ under any reward $\cR$ is
\[
J_\cR(\pi) = J_\cR(\pi_0) + c_\pi \langle \cR,\; \cF^{\pi_1} - \cF^{\pi_0} \rangle,
\]
which is affine in $c_\pi$. Since the $c_\pi$ are distinct, the ordering of $\Pi$ under $\cR$ is determined entirely by the sign of $\langle \cR, \cF^{\pi_1} - \cF^{\pi_0} \rangle$. Given a non-trivial $\cR$, any $\cR'$ either has the same sign (same ordering, so equivalent), the opposite sign (reversed ordering, so hackable), or no sign ($\cR'$ trivial on $\Pi$). A non-trivial, non-equivalent pair therefore cannot be unhackable on $\Pi$.
\end{proofEnd}

 Unfortunately, even assuming all of the regularity conditions above together with those of \hyperref[thm:skalse-thm2]{Theorem~S2}, this is still insufficient to guarantee the existence of a non-trivial, non-equivalent unexploitable transition model pair on finite policy sets (see \cref{ce:finite-exploitation} in \cref{app:proofs}). We suspect that the finite-policy-set case admits no clean existence guarantee for unexploitability, because the nonlinear dependence of $J$ on $\cT$ makes the geometry of the induced value orderings fundamentally richer than in the reward hacking setting.

\begin{theoremEnd}[all end]{counterexample}
\label{ce:finite-exploitation}
Consider the task $(\cS, \cA, \_, d_0, \cR, \gamma)$ with $\cS = \{s_0, s_1\}$, $\cA = \{a_0, a_1\}$, $d_0 = (1/2,\, 1/2)$, $\cR(s, a) = \mathbbm{1}\{s = s_0,\, a = a_0\} + \mathbbm{1}\{s = s_1,\, a = a_1\}$, and any $\gamma \in (0, 1)$. Let $\dPi = \{\pi_0, \pi_1, \pi_2\}$ where $\pi_0$ plays $a_0$ at both states, $\pi_1$ plays $a_0$ at $s_0$ and $a_1$ at $s_1$, and $\pi_2$ plays $a_1$ at $s_0$ and $a_0$ at $s_1$.
 
In this task, $\gamma > 0$, $\cR$ is non-trivial and state-action-dependent, and the visit counts of $\dPi$ are non-collinear under any transition model: $\cF^{\pi_0}$, $\cF^{\pi_1}$, and $\cF^{\pi_2}$ are supported on $\{(s_0, a_0), (s_1, a_0)\}$, $\{(s_0, a_0), (s_1, a_1)\}$, and $\{(s_0, a_1), (s_1, a_0)\}$ respectively, and no affine combination can reconcile the disjoint supports.
 
Under $\pi_1$, we see that $r^{\pi_1}(s) = 1$ for both states, so $J_\cT(\pi_1) = 1/(1 - \gamma)$ for every $\cT$. Under $\pi_2$, we see that $r^{\pi_2}(s) = 0$ for both states, so $J_\cT(\pi_2) = 0$ for every $\cT$. Under $\pi_0$, $r^{\pi_0}(s_0) = 1$ and $r^{\pi_0}(s_1) = 0$, so $J_\cT(\pi_0) = \mu^{\pi_0}_\cT(s_0)$ where $\mu^{\pi_0}_\cT(s)$ is the discounted visitation count of state $s$. Since $d_0(s_0) = 1/2$, the $t = 0$ contribution gives $\mu^{\pi_0}_\cT(s_0) \geq 1/2 > 0$. Since $d_0(s_1) = 1/2$, we also have $\mu^{\pi_0}_\cT(s_1) \geq 1/2$, so $\mu^{\pi_0}_\cT(s_0) = 1/(1-\gamma) - \mu^{\pi_0}_\cT(s_1) < 1/(1 - \gamma)$. Therefore, for every $\cT$,
\[
    \frac{1}{1 - \gamma} = J_\cT(\pi_1) > J_\cT(\pi_0) > 0 = J_\cT(\pi_2),
\]
and every pair of transition models is equivalent on $\dPi$.
\end{theoremEnd}

\begin{question}
\label{conj:finite}
Are there necessary and sufficient conditions on a task $(\cS, \cA, \_, d_0, \cR, \gamma)$ and finite policy set that guarantee the existence of a non-trivial, non-equivalent, unexploitable pair of transition models?
\end{question}

\subsection{Unifying model exploitation and reward hacking}
\label{sec:unifying}

\looseness=-1
We next return to our main question: \textit{can exploitability be avoided by restricting the policy set?}
As discussed in \cref{sec:Distinguishing} above, \hyperref[thm:skalse-thm1]{Theorem~S1} of \citet{skalse2022defining} shows that the answer is no for reward hacking on any policy set containing an open subset, but the result does not directly apply to exploitation. Surprisingly, this does not mean that reward hacking has nothing to tell us about model exploitation. In this section, we first show that every instance of model exploitation can be reduced to an instance of reward hacking, then develop this insight into a unified theory of exploitation and hacking.

\begin{theoremEnd}{proposition}[Exploitation implies hacking]
\label{prop:exploitation-implies-hacking}
For any exploitable pair $(\cT, \cT')$ relative to a policy set $\Pi$ and a task $(\cS, \cA, \_, d_0, \cR, \gamma)$, there exists a reward function $\cR'$ such that $(\cR, \cR')$ is hackable relative to $\Pi$ and the environment $(\cS, \cA, \cT, d_0, \_, \gamma)$.
\end{theoremEnd}
\begin{proofEnd}
Since $(\cT, \cT')$ is exploitable, there exist $\pi, \pi' \in \Pi$ with
$J_\cT(\pi) > J_\cT(\pi')$ and $J_{\cT'}(\pi') > J_{\cT'}(\pi)$.
Working in the environment $(\cS, \cA, \cT, d_0, \_, \gamma)$, the original reward $\cR$ satisfies $J_\cR(\pi) = J_\cT(\pi) > J_\cT(\pi') = J_\cR(\pi')$. Define $\cR' = \cF_\cT^{\pi'} - \cF_\cT^{\pi}$. Since $J_\cT(\pi) \neq J_\cT(\pi')$, the visit counts $\cF_\cT^\pi$ and $\cF_\cT^{\pi'}$ are distinct, so $\cR' \neq 0$. Then
\[
J_{\cR'}(\pi') - J_{\cR'}(\pi) = \langle \cR', \cF_\cT^{\pi'} - \cF_\cT^{\pi} \rangle = \lVert \cF_\cT^{\pi'} - \cF_\cT^{\pi} \rVert^2 > 0.
\]
Therefore $J_\cR(\pi) > J_\cR(\pi')$ and $J_{\cR'}(\pi') > J_{\cR'}(\pi)$, so $(\cR, \cR')$ is hackable.
\end{proofEnd}

While the result is perhaps surprising, the proof is a short construction: set $\cR'$ to be the difference in visit counts $\cF_\cT^{\pi'} - \cF_\cT^{\pi}$ for some exploiting policy pair $\pi,\pi' \in \Pi$. Interestingly, the converse does not hold.

\begin{counterexample}[Hacking may not imply exploitation]
\label{ce:hacking-exploition}
Consider any MDP with $|\cA| \geq 2$ and $\gamma > 0$. Let $\cR(s,a) = \mathbbm{1}\{a = a_1\}$ and $\cR'(s,a) = \mathbbm{1}\{a = a_2\}$, and let $\pi$ always select $a_1$ and $\pi'$ always select $a_2$. Then $(\cR, \cR')$ is hackable, since $J_\cR(\pi) = 1/(1-\gamma) > 0 = J_\cR(\pi')$ and $J_{\cR'}(\pi') = 1/(1-\gamma) > 0 = J_{\cR'}(\pi)$, but no $\cT'$ exists such that $(\cT, \cT')$ is exploitable since all transition models are equivalent.
\end{counterexample}

\Cref{prop:exploitation-implies-hacking} is unassuming by itself, but it hints at something more fundamental: a unified theory of model exploitation and reward hacking may exist \textit{despite} the anatomical differences between reward functions and transition models. If we can find such a theory, we may also discover an answer to our initial question.

We begin by observing that both model exploitation and reward hacking can be expressed as inversions of the value function. Formally, we say two value functions $J_1$ and $J_2$ admit a \textbf{value inversion} on $\Pi$ if there exist $\pi, \pi' \in \Pi$ such that $J_1(\pi) > J_1(\pi')$ and $J_2(\pi') > J_2(\pi)$. Reasoning directly about value functions therefore gives a natural starting point for a unified treatment. Unfortunately, the difficulties that make transition models unwieldy relative to reward functions (\cref{sec:anatomy}) also apply to the value function, so arguments that rely on $J$ being linear in the reward, like the one behind \hyperref[thm:skalse-thm1]{Theorem~S1}, still do not apply. The key insight to breaking past this obstacle is that the value function is rational in policy space.

\begin{theoremEnd}{lemma}[Value function rationality]
\label{lem:rational}
For any MDP with finite $\cS$ and $\cA$, the value function $J(\pi)$ is rational in $\pi$. In particular, $J$ is real-analytic on $\Pi^+$.
\end{theoremEnd}
\begin{proofEnd}
    Write $J(\pi) = \sum_s d_0(s)\, V^\pi(s)$; it suffices to show each $V^\pi(s)$ is rational in $\pi$. Let $\cR^\pi(s) = \sum_a \pi(a|s)\,\cR(s,a)$, $P^\pi(s'|s) = \sum_a \pi(a|s)\,\cT(s'|s,a)$, and $A = I - \gamma P^\pi$. The Bellman equation in matrix form is $A V^\pi = \cR^\pi$. Since $P^\pi$ is a probability matrix and $\gamma \in [0, 1)$, $A$ is invertible and $V^\pi = A^{-1} \cR^\pi$. By Cramer's rule, $V^\pi$ can then be expressed as a ratio of determinants and is thus rational as required.
\end{proofEnd}

Rational functions are typically easier to reason about compared to nonlinear functions in general because they have a host of desirable properties \citep{rudin1974real}. The two properties that are useful for our purposes are analyticity and rigidity. The former provides access to an arsenal of powerful results from analysis and differential geometry, while the latter allows us to promote properties from open subsets to the supersets containing them. 

More concretely, from these two properties, we can derive the following lemmas, which relate the gradients of value functions to the existence (or lack thereof) of inversions and give a simple geometric picture: at every point in policy space, the relationship between $\nJ_1$ and $\nJ_2$ determines whether the two value functions can disagree on which policy is better.\footnote{From now on, we identify the tangent space to the policy space $\Delta(\cA)^{|\cS|}$ at any $\pi$ with $\mathbb{R}^{|\cS|(|\cA|-1)}$, although the identification is not canonical. We can thus write $\nabla_vJ(\pi)$ for $v\in \mathbb{R}^{|\cS|(|\cA|-1)}$ an infinitesimal perturbation of $\pi$.}

\begin{figure}[t]
    \vspace*{-1em}
  \centering
  \includegraphics[width=\textwidth]{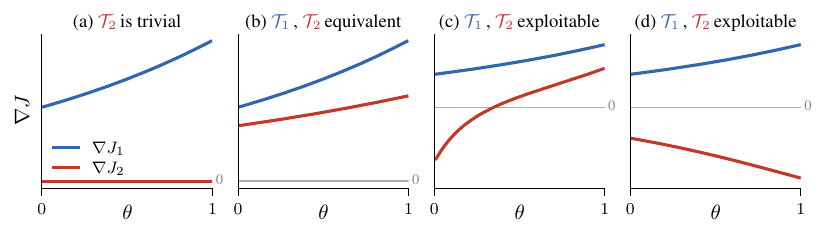}\\[-1em]
  \caption{Gradients for the value curves in \cref{fig:taxonomy}. Exploitation exists if the gradients have opposite signs at any $\theta$.
  \textbf{(a)}~\textcolor{grad2}{$\nabla J_2$} vanishes (trivial).
  \textbf{(b)}~Both gradients positive: proportional gradients force
  equivalence (\cref{lem:equivalence}).
  \textbf{(c, d)}~The gradients have opposite signs,
yielding exploitation (\cref{lem:exploitation}). In~(c) this
occurs roughly on $\theta \in [0, 0.3]$; in~(d) it holds everywhere.}
  \label{fig:gradients}
\end{figure}

\begin{theoremEnd}{lemma}[Local inversion]
\label{lem:exploitation}
Let $\Pi$ be open. If there is a direction $v \in \mathbb{R}^{|\cS|(|\cA| - 1)}$ along which $J_1$ increases and $J_2$ decreases at some $\pi \in \Pi$ (i.e., $\nabla_vJ_2(\pi)<0<\nabla_vJ_1(\pi)$), then $J_1$ and $J_2$ admit an inversion on $\Pi$.
\end{theoremEnd}
\begin{proofEnd}
Define $\pi^\pm = \pi \pm \varepsilon v$ for 
$\varepsilon > 0$. Since each $J_i$ is analytic (\cref{lem:rational}), 
first-order Taylor expansion gives $J_i(\pi^\pm) = J_i(\pi) \pm \varepsilon \langle \nJ_i(\pi), v \rangle + O(\varepsilon^2)$ and subtracting yields
\begin{align*}
J_1(\pi^+) - J_1(\pi^-) &= 2\varepsilon \langle \nJ_1(\pi), v \rangle + O(\varepsilon^2) \\
J_2(\pi^-) - J_2(\pi^+) &= -2\varepsilon \langle \nJ_2(\pi), v \rangle + O(\varepsilon^2)
\end{align*}
Since $\langle \nJ_1(\pi), v \rangle > 0$ and $\langle \nJ_2(\pi), v \rangle < 0$, we can make $\varepsilon$ small enough so that both differences are positive and $\pi^\pm \in \Pi$ (openness). Thus, $\cT_1$ and $\cT_2$ are exploitable.
\end{proofEnd}

\begin{theoremEnd}[all end]{proposition}[Connected level sets]
    \label{prop:connected-level-sets}
    For every transition model $\cT$ and $c \in \mathbb{R}$, the level set $\{J_\cT = c\} \cap \Pi^+$ is connected.
\end{theoremEnd}
\begin{proofEnd}
    Since the map $\pi \mapsto \cF^\pi$ is a homeomorphism on $\Pi^+$ \citep[Lemma~1]{skalse2022defining}, it suffices to show connectedness in visit-count space. The set $\cF(\Pi^+)$ is the interior of a polyhedron (defined by non-negativity and Bellman flow constraints), hence convex. Since $J_\cT(f) = \langle \cR, f \rangle$ is linear, $\{\langle \cR, f \rangle = c\} \cap \cF(\Pi^+)$ is a hyperplane intersected with a convex set, which is convex and hence connected.
\end{proofEnd}

\begin{theoremEnd}{lemma}[Global equivalence]
\label{lem:equivalence}
Let $\tilde{\Pi} \subseteq \Pi$ be open. If $J_1$ and $J_2$ are non-trivial and have positively proportional gradients wherever both are nonzero on $\tilde{\Pi}$, then $J_1$ and $J_2$ are equivalent on $\Pi$.
\end{theoremEnd}
\begin{proofEnd}
    We show that positive proportionality of gradients on $\tilde{\Pi}$ forces equivalence on all of $\Pi$. The idea is to first extend  proportionality from $\tilde{\Pi}$ to all of $\Pi^+$ using analyticity and connectedness, then construct a strictly increasing reparametrization $J_2 = \varphi \circ J_1$ on $\Pi^+$, and finally extend this to $\Pi$ by density.

    We begin by extending the proportionality to $\Pi^+$. We want to establish that $\nJ_2 = \lambda \nJ_1$ with $\lambda > 0$ on all of $\Pi^+$. This requires showing that on $\Pi^+$ (i) the gradients are linearly dependent, (ii) the gradients are nonzero, and (iii) $\lambda$ is positive.

    For (i), the Gram determinant of $\nJ_1$ and $\nJ_2$ is real-analytic on the domain $\Pi^+$ (\cref{lem:rational}) and vanishes on the open subset $\tilde{\Pi}$, so it vanishes identically on $\Pi^+$. The gradients are therefore linearly dependent everywhere on $\Pi^+$.

    For (ii), write $f_i\colon \pi \mapsto \cF^\pi_{\cT_i}$. The chain rule gives $\nJ_i(\pi) = Df_i(\pi)^\top \cR_i$. The map $f_i$ is a homeomorphism from $\Pi^+$ onto its image \citep[Lemma~1]{skalse2022defining}. Both $f_i$ and its inverse $\cF \mapsto \pi(a|s) = \cF(s,a)/\sum_{a'}\cF(s,a')$ are smooth (the former by rationality, the latter since all visitations are positive on $\Pi^+$), so $f_i$ is in fact a diffeomorphism onto its image (in particular, a smooth immersion) and $Df_i(\pi)$ has full column rank. Non-triviality of $J_i$ ensures $\cR$ is not orthogonal to the column space of $Df_i(\pi)$, since otherwise $J_i$ would be constant. Therefore $\nJ_i(\pi) = Df_i(\pi)^\top \cR \neq 0$ on $\Pi^+$.

    For (iii), since both gradients are nonzero, we may write $\nJ_2 = \lambda \nJ_1$ where $\lambda(\pi) = \langle \nJ_2, \nJ_1 \rangle / \|\nJ_1\|^2$ is continuous and nonzero on $\Pi^+$. Since $\Pi^+$ is connected, $\lambda$ has constant sign. Since $\lambda > 0$ on $\tilde{\Pi}$ by hypothesis, $\lambda > 0$ on $\Pi^+$.

    Having extended positive proportionality of the gradients to $\Pi^+$, we now construct a strictly increasing $\varphi$ with $J_2 = \varphi \circ J_1$ on $\Pi^+$. Such a $\varphi$ exists if and only if $J_1(\pi) = J_1(\pi')$ implies $J_2(\pi) = J_2(\pi')$ for all $\pi, \pi' \in \Pi^+$. Pick any such pair $\pi,\pi'$. Both policies lie in the same level set of $J_1$ in $\Pi^+$, and the level set is connected (\cref{prop:connected-level-sets}), so there is a path $\alpha: [0,1] \to \Pi^+$ in the level set with $\alpha(0) = \pi$ and $\alpha(1) = \pi'$. Since $J_1(\alpha(t))$ is constant, differentiating gives $\langle \nJ_1(\alpha(t)), \frac{d}{dt}\alpha(t) \rangle = 0$. Since $\nJ_2 = \lambda \nJ_1$, we also have $\langle \nJ_2(\alpha(t)), \frac{d}{dt}\alpha(t) \rangle = 0$, so $J_2(\alpha(t))$ is constant on $[0,1]$, giving $J_2(\pi) = J_2(\alpha(0)) = J_2(\alpha(1)) = J_2(\pi')$.
    
    It remains to show that such a $\varphi$ is strictly increasing on the image of $J_1$. Fix $c \in J_1(\Pi^+)$ and pick $\pi_0$ with $J_1(\pi_0) = c$. Since $\nJ_1(\pi_0) \neq 0$, we can choose a direction $v$ with $\langle \nJ_1(\pi_0), v \rangle \neq 0$ and define $h(x) = J_1(\pi_0 + xv)$. Then $h(0) = c$ and $\frac{d}{dx}h(0) \neq 0$, so by the inverse function theorem $h$ is locally invertible. Setting $\beta(t) = \pi_0 + h^{-1}(t)v$ gives a smooth curve with $J_1(\beta(t)) = h(h^{-1}(t)) = t$ near $c$. Then $\varphi(t) = J_2(\beta(t))$, so
    \[
        \frac{d}{dt}\varphi(t) = \langle \nJ_2(\beta(t)), \tfrac{d}{dt}\beta(t) \rangle = \lambda(\beta(t))\langle \nJ_1(\beta(t)), \tfrac{d}{dt}\beta(t) \rangle = \lambda(\beta(t)),
    \]
    where the last equality uses $\frac{d}{dt}J_1(\beta(t)) = 1$. Since $\lambda > 0$, $\varphi$ is locally strictly increasing at every point of $J_1(\Pi^+)$, hence strictly increasing on $J_1(\Pi^+)$.

    Finally, we use the density of $\Pi^+$ in the space of all stationary policies to extend to $\Pi$. Define $\bar{\varphi}(u) = \lim_{n \to \infty} \varphi(u_n)$ for any sequence $(u_n) \subseteq J_1(\Pi^+)$ with $u_n \to u$; this is well-defined because $\varphi$ is bounded and monotone. By continuity of $J_2$, we have $\bar{\varphi}(J_1(\pi)) = \lim_{n \to \infty} \varphi(J_1(\pi_n)) = \lim_{n \to \infty} J_2(\pi_n) = J_2(\pi)$ for any $\pi \in \Pi$ and any sequence $(\pi_n) \subseteq \Pi^+$ with $\pi_n \to \pi$. The extension $\bar{\varphi}$ is strictly increasing: given $a < d$ in $J_1(\Pi)$, density of $J_1(\Pi^+)$ provides $b, c \in J_1(\Pi^+)$ with $a < b < c < d$, giving $\bar{\varphi}(a) \leq \varphi(b) < \varphi(c) \leq \bar{\varphi}(d)$. So for any $\pi, \pi' \in \Pi$, $J_1(\pi) > J_1(\pi')$ implies $J_2(\pi) = \bar{\varphi}(J_1(\pi)) > \bar{\varphi}(J_1(\pi')) = J_2(\pi')$, and $J_1$ and $J_2$ induce the same ordering on $\Pi$, i.e., they are equivalent on $\Pi$.
\end{proofEnd}

The intuition for \cref{lem:exploitation} is as follows. If there is a direction in policy space that improves performance under one model but decreases it under the other, then a step in that direction leads to two policies that form an exploiting pair.

\Cref{lem:equivalence} has a similar intuition. If every direction that improves one value function also improves the other, then the two can never disagree on which of any two policies is better, so they must induce the same ordering.\footnote{Proportionality here is pointwise. At each $\pi$ the two gradients must point in the same direction, but the ratio $\lambda(\pi)$ between them may vary across the policy space.} While the result is simple, the formal proof is surprisingly involved and relies heavily on value function rationality.

Together, these two lemmas cover every possible relationship between $\nJ_1$ and $\nJ_2$ (\cref{fig:gradients}). At every point in policy space, the gradients are either linearly independent, antiparallel, or positively proportional. The first two cases produce an inversion via \cref{lem:exploitation}; the third forces equivalence via \cref{lem:equivalence}. This gives our main characterization.

\begin{theoremEnd}{theorem}[Value inversions]
\label{thm:inversion}
Let $J_1$ and $J_2$ be non-trivial, non-equivalent value functions. If the policy set contains an open subset, then it admits a value inversion for $J_1$ and $J_2$.
\end{theoremEnd}
\begin{proofEnd}
    Let $\tilde{\Pi} \subseteq \Pi$ be open, and let $\cT_1$, $\cT_2$ be non-trivial and non-equivalent on $\Pi$. Since $J_1$ and $J_2$ are rational (\cref{lem:rational}), their gradients $\nJ_1$ and $\nJ_2$ exist. If these gradients are linearly independent at some $\pi_0 \in \tilde{\Pi}$, we can solve $\langle \nJ_1(\pi_0), v \rangle = 1$ and $\langle \nJ_2(\pi_0), v \rangle = -1$ (a consistent linear system in $|\cS|(|\cA|-1) \geq 2$ unknowns), and \cref{lem:exploitation} gives exploitation. If instead they are everywhere linearly dependent on $\tilde{\Pi}$, we may write $\nJ_2 = \lambda \nJ_1$ for a scalar-valued function $\lambda$ wherever both gradients are nonzero (such points exist since each $\nJ_i$ is analytic and not identically zero by non-triviality, so it is nonzero on a subset open and dense in $\tilde{\Pi}$). If $\lambda(\pi_0) < 0$ at some such point, setting $v = \nJ_1(\pi_0)$ gives directional derivatives of opposite sign, and \cref{lem:exploitation} again gives exploitation. The remaining possibility is that $\lambda > 0$ wherever both gradients are nonzero on $\tilde{\Pi}$, but then \cref{lem:equivalence} forces $\cT_1$ and $\cT_2$ to be equivalent on $\Pi$, contradicting our assumption.
\end{proofEnd}

The theorem immediately yields the desired characterization for model exploitation.

\begin{corollary}[Imperfect world models are exploitable]
\label{cor:exploitability}
On any policy set containing an open subset, every non-trivial, non-equivalent pair of transition models is exploitable.
\end{corollary}

In practice, one might hope that \cref{cor:exploitability} only applies to high-entropy or poorly performing policies that no reasonable learning algorithm would ever produce. Unfortunately, \cref{thm:inversion} still applies both to the set of $\varepsilon$-suboptimal policies (those $\pi$ with $J(\pi) \geq \sup_{\pi'} J(\pi') - \varepsilon$ for $\varepsilon > 0$) and to the set of $\delta$-deterministic policies (those $\pi$ with $\max_a \pi(a\mid s) \geq \delta$ for $\delta < 1$ and every $s \in \cS$).

\begin{corollary}[Common exploitable policy sets]
\label{cor:subsets}
Every non-trivial, non-equivalent pair of transition models is exploitable on (i) the set of all stationary policies, (ii) the set of all $\varepsilon$-suboptimal policies ($\varepsilon > 0$), and (iii) the set of all $\delta$-deterministic policies ($\delta < 1$).
\end{corollary}

\Cref{thm:inversion} implies characterizations of both reward hacking and model exploitation. Since reward hacking is also a value inversion, \cref{thm:inversion} also recovers \hyperref[thm:skalse-thm1]{Theorem~S1} of \citet{skalse2022defining} \emph{a fortiori}.

\begin{corollary}[{{\hyperref[thm:skalse-thm1]{Theorem~S1}} of \citet{skalse2022defining}}]
\label{cor:skalse}
On any policy set containing an open subset, every non-trivial, non-equivalent pair of reward functions is hackable.
\end{corollary}

\subsection{Relaxation}

\begin{figure}[t]
  \centering
  \includegraphics{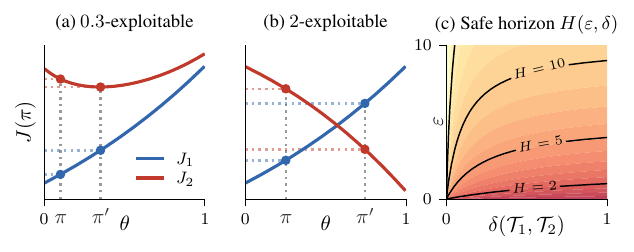}\\[-1em]
  \caption{Examples of $\varepsilon$-exploitability and a contour plot for the safe horizon.
  \textbf{(a, b)}~The exploitable transition model pairs from
  \cref{fig:taxonomy}(c, d), annotated with exploiting policies.
  The pair in~(a) is $0.3$-exploitable and the pair in~(b) is
  $2$-exploitable.
  \textbf{(c)}~The safe horizon $H(\varepsilon, \delta)$ from
  \cref{thm:safe-horizon}. Each contour marks the largest effective horizon
  $1/(1-\gamma)$ under which $\varepsilon$-unexploitability is
  guaranteed for a given tolerance $\varepsilon$ and maximum total variation
  distance $\delta(\cT_1, \cT_2)$.}
  \label{fig:safe-horizon}
\end{figure}

We now introduce a relaxed notion of exploitability to pursue two goals. \textit{First}, practitioners may not care about small-value inversions in real-world applications, so we seek a more graded measure of exploitability that distinguishes between negligible and catastrophic value inversions. \textit{Second}, we want to answer our initial question about when exploitation is avoidable, even if only approximately. Thus far, we have seen that interesting unexploitability is impossible on large policy sets (\cref{cor:exploitability}) and resists characterization on finite ones (\cref{sec:finite}). A weaker definition, in contrast, may be more likely to admit positive existence conditions than our usual notion of exploitation.

\begin{definition}[$\varepsilon$-Exploitation]
\label{def:eps-exploitability}
Transition models $\cT$ and $\cT'$ are \textbf{$\varepsilon$-exploitable} relative to a policy set $\Pi$ and a task $(\cS, \cA, \_, d_0, \cR, \gamma)$ if there exist $\pi, \pi' \in \Pi$ such that
\begin{equation*}
    J_\cT(\pi) - J_\cT(\pi') > \varepsilon \quad \text{and} \quad J_{\cT'}(\pi') - J_{\cT'}(\pi) > \varepsilon,
\end{equation*}
otherwise they are \textbf{$\varepsilon$-unexploitable}.
\end{definition}

Intuitively, $\varepsilon$-unexploitability says that while small inversions may exist, $\cT$ and $\cT'$ may never confidently disagree about how to rank a policy pair. This definition immediately satisfies our first goal, because it generalizes exploitation (recovering \cref{def:exploitability} when $\varepsilon = 0$) and is a continuous rather than binary measure of ``exploitability'' (i.e., via the minimum $\varepsilon$ for which a pair of transition models remains $\varepsilon$-unexploitable).

To address our second goal and establish sufficient conditions for $\varepsilon$-unexploitability, we derive the following theorem.

\begin{theoremEnd}{theorem}[Safe horizon]
\label{thm:safe-horizon}
Let $\cR \colon \cS \times \cA \to [0,1]$ and let $\delta = \frac{1}{2} \max_{s,a} \| \cT(\cdot|s,a) - \cT'(\cdot|s,a) \|_1$ be the total variation distance between distinct $\cT$ and $\cT'$. For any $\varepsilon > 0$, define the safe horizon
\begin{equation*}
    H(\varepsilon, \delta) = \frac{(1+\varepsilon) + \sqrt{(1-\varepsilon)^2 + 4\varepsilon/\delta}}{2}.
\end{equation*}
Every pair $(\cT, \cT')$ is $\varepsilon$-unexploitable on every policy set $\Pi$ whenever $1/(1-\gamma) \leq H(\varepsilon,\delta)$. Furthermore, this bound is tight.
\end{theoremEnd}
\begin{proofEnd}
Write $h = 1/(1-\gamma)$. The tight simulation lemma \citep{lobeloptimal} with $\varepsilon_\cR = 0$ gives $|J_1(\pi) - J_2(\pi)| \leq B$ for every $\pi$, where $B=h-\frac{h}{1-(1-h)\delta}$.

Suppose $(\cT_1, \cT_2)$ is $\varepsilon$-exploitable via $\pi, \pi'$. Add the two defining inequalities:
\begin{equation*}
    \underbrace{(J_1(\pi) - J_2(\pi))}_{\leq\, B} + \underbrace{(J_2(\pi') - J_1(\pi'))}_{\leq\, B} \;\;>\;\; 2\varepsilon,
\end{equation*}
so $B > \varepsilon$. Contrapositively, $B \leq \varepsilon$ implies $\varepsilon$-unexploitability.
 
It is elementary to show that for $\delta,\varepsilon>0$ there is a unique $h = H(\varepsilon,\delta)>1$ at which $B = \varepsilon$, and $h < H(\varepsilon,\delta)$ implies $B < \varepsilon$. Solving $B = \varepsilon$ reduces to a quadratic whose positive root is the $H(\varepsilon, \delta)$ in the theorem statement.
 
For tightness, take $\cS = \{0, 1\}$, $\cA = \{0, 1\}$, $d_0(0) = 1$, $\cR(s,a) = \mathbbm{1}\{s = 0\}$ and let state $1$ be absorbing. Under $\cT_1$, action $0$ self-loops at state $0$ and action $1$ transitions to state $1$ with probability $\delta$. Under $\cT_2$, the actions swap, so action $1$ self-loops at state $0$ and action $0$ transitions to state $1$ with probability $\delta$. Then $J_1(\pi_0) = J_2(\pi_1) = h$ and $J_1(\pi_1) = J_2(\pi_0) = h/(1+(h{-}1)\delta)$, so both gaps equal $B$ exactly, and the pair is $\varepsilon$-exploitable whenever $h > H$.
\end{proofEnd}


The safe horizon $H(\varepsilon, \delta)$ is the longest effective horizon $1/(1-\gamma)$ at which $\varepsilon$-unexploitability is guaranteed.\footnote{\Cref{thm:safe-horizon} uses a reward in $[0,1]$ for notational clarity but easily generalizes to all bounded reward functions. Explicitly, if $\cR$ takes values in $[a, b]$ with $a < b$, then $\tilde{\cR} = (\cR - a)/(b - a)$ takes values in $[0,1]$ and every value gap under $\cR$ is $(b-a)$ times the corresponding gap under $\tilde{\cR}$. Hence every pair $(\cT, \cT')$ is $\varepsilon(b-a)$-unexploitable under $\cR$ whenever $1/(1-\gamma) \leq H(\varepsilon, \delta)$.} Its dependence on both arguments aligns with standard intuitions in model-based RL. Increasing $\delta$ lowers $H$: a less accurate model exhausts its error budget sooner, so safe planning requires a shorter horizon. This is exactly in line with the main result by \cite{jiang2015dependence}, who show that a shorter effective horizon should be used when data is scarce. In this sense, we corroborate and expand their findings through the lens of $\varepsilon$-exploitability.

Increasing $\varepsilon$ raises $H$: a designer who tolerates larger inversions can plan further ahead. At the extremes, $H \to \infty$ as $\delta \to 0$ (a perfect model is never exploitable) and $H \to 1 + \varepsilon$ as $\delta \to 1$ (a maximally inaccurate model is safe only near single-step planning). This formalizes the widely held intuition that model errors compound over the planning horizon \citep{ross2011reduction, talvitie2014model}, and provides a closed-form expression for exactly how much horizon a given level of model error can afford (\cref{fig:safe-horizon}c).

The closed form of $H$ is exact but unwieldy. A looser but more memorable sufficient condition follows from dropping lower-order terms.

\begin{theoremEnd}{corollary}[Square-root heuristic]
\label{cor:sqrt-rule}
Let $\cR \colon \cS \times \cA \to [0,1]$ and let $\delta = \frac{1}{2} \max_{s,a} \| \cT(\cdot|s,a) - \cT'(\cdot|s,a) \|_1 > 0$. For any $\varepsilon > 0$, the pair $(\cT, \cT')$ is $\varepsilon$-unexploitable whenever $1 / (1 - \gamma) < \sqrt{\varepsilon / \delta}$.
\end{theoremEnd}
\begin{proofEnd}
It suffices to show $\sqrt{\varepsilon/\delta} \leq H(\varepsilon, \delta)$, since then $1/(1-\gamma) < \sqrt{\varepsilon/\delta}$ implies $1/(1-\gamma) < H(\varepsilon, \delta)$ and \cref{thm:safe-horizon} gives $\varepsilon$-unexploitability. Since $(1-\varepsilon)^2 \geq 0$,
\[
    \sqrt{(1-\varepsilon)^2 + \frac{4\varepsilon}{\delta}} \;\geq\; 2\sqrt{\frac{\varepsilon}{\delta}},
\]
so $H(\varepsilon, \delta) \geq \frac{1+\varepsilon}{2} + \sqrt{\varepsilon/\delta} > \sqrt{\varepsilon/\delta}$.
\end{proofEnd}

\Cref{cor:sqrt-rule} gives a quick sense of the magnitude of possible inversions: at $\gamma = 0.9$ the effective horizon is $10$, so the largest inversion the bound allows scales as $100\delta$; at $\gamma = 0.99$ it scales as $10{,}000\delta$. The bound in \Cref{thm:safe-horizon} is tight, but the approximate bound in \Cref{cor:sqrt-rule} is not, so the actual inversion may be much smaller. Still, it gives an immediate sense of how model error and planning horizon interact without requiring further analysis. For the transition model pairs in \cref{fig:taxonomy}, we compare both bounds against the actual exploitation gaps in \cref{app:bound-comparison}.

\section{Related work}

\paragraph{History and etymology of model exploitation.} The term ``model exploitation'' has no single origin in RL. Adjacent concepts, such as ``model bias'' and ``compounding error,'' have appeared in the literature since at least the early 2010s \citep{deisenroth2011pilco, ross2011reduction}, themselves inherited from control theory \citep{schneider1996exploiting, kappen2005path}. As deep RL grew in popularity toward the end of that decade, practitioners turned to sample-efficient model-based RL (MBRL) to meet rising data demands. Learned models, however, are uncertain where data is scarce, and unlike in model-free methods, no trial-and-error mechanism exists to correct the resulting errors. Researchers began describing policies as ``exploiting'' these uncertain regions \citep{kurutach2018model}, though this usage of ``exploit'' also predates deep RL (cf.\ \citealt{schneider1996exploiting}). The set phrase ``model exploitation'' congealed only more recently, following \citet{ha2018worldmodels} and \citet{janner2019trust}, and has since motivated a lineage of conservative offline MBRL methods \citep{yu2020mopo, kidambi2020morel}.

\paragraph{Other adjacent phenomena.} Model exploitation overlaps with several independent concepts. Planner overfitting \citep{arumugam2018mitigating,jiang2015dependence} refers to a planner discovering spurious shortcuts in a learned model and develops regularization-based remedies. \cite{jiang2015dependence} were the first to formalize and isolate this phenomenon by highlighting the role that the effective horizon plays in regularizing a learned transition model when data is scarce. The sim-to-real gap \citep{jakobi1995noise, tobin2017domain} describes the same failure when transferring controllers from simulated to physical environments, typically addressed through domain randomization or system identification. Objective robustness \citep{koch2021objective} studies agents that achieve high reward in training environments but pursue unintended objectives at deployment, which can be viewed as exploitation of a training-time model of the deployment distribution. Lastly, as discussed in the introduction, the value equivalence principle \citep{grimm2020value,grimm2021proper} provides a framing of when two models can be well-understood as equivalent in terms of \textit{value preservation} under the Bellman operator. This principle is closely tied to our notion of equivalent transition models (\Cref{fig:taxonomy}).

\paragraph{MDP extensions.} Several lines of work extend the MDP formalism to defend against imperfect transition models, typically by modifying the optimization problem rather than studying the misspecification itself. Robust MDPs \citep{iyengar2005robust, nilim2005robust} represent transition uncertainty through uncertainty sets and optimize worst-case value, a cardinal safety criterion rather than the ordinal one we propose. MOReL \citep{kidambi2020morel} constructs a pessimistic MDP by introducing absorbing states at high-uncertainty transitions, and provides strong suboptimality guarantees for the resulting policy, but does not formalize the notion of exploitation it defends against. The family of approaches to model-based RL that use the principle of \textit{optimism under uncertainty}, such as $E^3$ \citep{kearns2002near} and R-Max \citep{brafman2002r}, drives exploration toward uncertain regions rather than away from them. All three modify the MDP to cope with model uncertainty; our work instead characterizes when that uncertainty inverts the policy ordering.

\paragraph{World models.} A growing family of methods trains policies entirely inside learned dynamics models, from the Dreamer line \citep{hafner2019dream, dreamer2, dreamer3, dreamer4} to MuZero \citep{schrittwieser2020mastering} to recent JEPA-style predictive representations \citep{balestriero2025lejepa}. These methods measure model quality through predictive accuracy or downstream task performance, both of which are cardinal. Our results show that low prediction error does not preclude ordinal failures: a model can be accurate yet still invert the policy ordering induced by the true dynamics (\cref{thm:safe-horizon}).

\paragraph{Outside machine learning.} Optimizing against an imperfect model is a recurring failure mode across disciplines, so there is an extensive literature that exists outside machine learning. For example, in macroeconomics, the Lucas critique \citep{lucas1976econometric} warns that policy rules derived from an estimated model can fail when agents adapt to the policy itself, and robust control \citep{hansen2008robustness} responds by optimizing against worst-case model perturbations. In sim-to-real robotics, controllers optimized in simulation routinely exploit discrepancies with real-world physics \citep{jakobi1995noise, tobin2017domain}. Our formalization of exploitation as a value inversion applies in principle to any setting where a proxy model guides optimization over a structured decision space.

\section{Conclusion}
\label{sec:conclusion}

There is increasing interest in approaches to decision-making that combine a forward predictive model with planning. Although there is substantial literature on the design and analysis of such model-based algorithms, the emergence of world models that make predictions directly from rich low-level sensor streams pushes the boundaries of these methods. This calls for a fresh look at the questions of model robustness and safety of the computed policies. In this paper, we formalize the concept of model exploitation and study the conditions under which policies trained on empirical models may be safe. We find that in sufficiently rich policy classes, exploitability is inevitable. However, under a relaxed definition, we can establish conditions for unexploitability. We suggest that this could guide the development of new approaches to safe policy synthesis, as well as parallel methods for quantifying the quality of synthesized policies with respect to such safety objectives. 

\paragraph{Limitations and future work.}

We identify three main limitations of our work. \textit{First}, our results, like those in  \citet{skalse2022defining}, concern MDPs with finite state and action spaces. We defer extensions to continuous and partially observable MDPs (POMDPs) to future work, noting that solving the former may yield the latter (since every POMDP can be written as a continuous belief-state MDP \citep{aastrom1965optimal}). \textit{Second}, also like in \citet{skalse2022defining}, our notion of exploitation is binary and restrictive. In practice, it is unlikely that small, rare value inversions will pose meaningful safety concerns. Future work could refine our notion of exploitation by introducing a more granular measure of exploitability (e.g., deriving bounds using correlation metrics \citep{laidlaw2024correlated} or expanding upon $\varepsilon$-exploitation). \textit{Third}, there are, as of yet, no known closed-form necessary and sufficient conditions that guarantee unexploitability in finite policy sets. Although we question whether such simple conditions exist (\cref{conj:finite}), there is still area to explore. Researchers interested in pursuing this direction may, for instance, find it productive to begin work on special finite policy sets and specific MDPs (e.g., characterize exploitation on deterministic policies in linear MDPs). 

\begin{ack}
This work was supported by a UKRI Turing AI World Leading Researcher Fellowship on AI for Person-Centred and Teachable Autonomy (grant EP/Z534833/1). For the purpose of open access, the authors have applied a Creative Commons Attribution (CC BY) licence to any Author Accepted Manuscript version arising from this submission.
The authors declare no competing interests.
\end{ack}

\bibliographystyle{plainnat}
\bibliography{references}


\appendix
\crefalias{section}{appendix}

\newpage

\section{Notation}
\label{app:notation_table}

We first provide a table summarizing all relevant notation.

\captionsetup[table]{skip=10pt}
\begin{table}[ht!]
\def\arraystretch{1.3}
    \centering
    \begin{footnotesize}
    \begin{tabular}{@{}lll@{}}
        \toprule
        %
        \textit{Notation}&\textit{Meaning}&\textit{Definition} \\
        \midrule
        %
        $\cS$& State space& \\
        $\cA$& Action space& \\
        $\cT$& Transition model& $\cT : \cS \times \cA \to \Delta(\cS)$ \\
        $d_0$& Initial state distribution& $d_0 \in \Delta(\cS)$ \\
        $\cR$& Reward function& $\cR : \cS \times \cA \to \mathbb{R}$ \\
        $\gamma$& Discount factor& $\gamma \in [0,1)$ \vspace{4mm}\\
        %
        $\cM$& Markov decision process& $(\cS, \cA, \cT, d_0, \cR, \gamma)$ \\
        $\cM \setminus \cR$& Environment (MDP without reward)& $(\cS, \cA, \cT, d_0, \_, \gamma)$ \\
        $\cM \setminus \cT$& Task (MDP without transitions)& $(\cS, \cA, \_, d_0, \cR, \gamma)$ \vspace{4mm}\\
        %
        $\pi$& Stationary policy& $\pi : \cS \to \Delta(\cA)$ \\
        $\Pi$& Policy set& \\
        $\Pi^+$& Interior of all stationary policies& $\{\pi \in \Delta(\cA)^{|\cS|}: \pi(a|s) > 0\ \forall_{s \in \cS, a \in \cA}\}$ \vspace{4mm}\\
        %
        $\cF^\pi_\cT(s,a)$& Discounted visit counts& $\sum_{t=0}^\infty \gamma^t \mathbbm{1}(s_t {=} s, a_t {=} a)$ \\
        $J(\pi)$& Policy value& $\langle \cR, \cF^\pi_\cT \rangle$ \\
        $J_\cT(\pi)$& Value under transition model $\cT$& $\langle \cR, \cF^\pi_\cT \rangle$ \\
        $J_\cR(\pi)$& Value under reward function $\cR$& $\langle \cR, \cF^\pi \rangle$ \\
        $J_i(\pi)$& Value under $\cR_i$ or $\cT_i$& (context determines index) \\
        $V^\pi(s)$& State-value function& $\E_{\tau \sim \pi}[G(\tau) \mid s_0 = s]$ \\
        $r^\pi(s)$& Per-state expected reward& $\sum_a \pi(a|s)\, \cR(s, a)$ \vspace{4mm}\\
        %
        $(\cR, \cR')$ hackable& Reward hacking& $\exists_{\pi, \pi' \in \Pi}\ J_\cR(\pi) > J_\cR(\pi')$ and $J_{\cR'}(\pi') > J_{\cR'}(\pi)$ \\
        $(\cT, \cT')$ exploitable& Model exploitation& $\exists_{\pi, \pi' \in \Pi}\ J_\cT(\pi) > J_\cT(\pi')$ and $J_{\cT'}(\pi') > J_{\cT'}(\pi)$ \vspace{4mm}\\
        %
        Equivalent& Same policy ordering& $\forall_{\pi, \pi' \in \Pi}\ J_1(\pi) \geq J_1(\pi') \Leftrightarrow J_2(\pi) \geq J_2(\pi')$ \\
        Trivial& Constant value& $\forall_{\pi, \pi' \in \Pi}\ J(\pi) = J(\pi')$ \\
        \bottomrule
    \end{tabular}
    \end{footnotesize}
    \caption{\vspace{4mm} A summary of notation.}
    \label{tab:notation}
\end{table}

\section{Proofs}
\label{app:proofs}

\printProofs

\section{Relevant reward hacking theorems}
\label{app:skalse-theorems}

\Cref{thm:skalse-thm1,thm:skalse-thm2} are restated in \cref{sec:prior-hacking}. Note that \cref{cor:skalse} and \cref{thm:skalse-thm1} are contrapositives of the same fact: \cref{thm:skalse-thm1} says that any unhackable, non-trivial pair must be equivalent, while \cref{cor:skalse} says that any non-trivial, non-equivalent pair must be hackable. We use the latter phrasing to parallel our exploitability results.\footnote{See also \citet{nayebi2026barriers} for a related impossibility result.}

As shown in \cref{prop:collinear-visit-counts}, the original statement of \hyperref[thm:skalse-thm2]{Theorem~S2} in \citet{skalse2022defining} omits a needed hypothesis: the visit counts of $\dPi$ must not be collinear. Our statement in \cref{sec:prior-hacking} includes it. The correction is minor and the strengthened condition is easily satisfied in practice; the core insights of \citet{skalse2022defining} are unaffected.

\begin{proof}[Proof of \cref{thm:skalse-thm2} (corrected)]
The proof is identical to that of \citet{skalse2022defining}. The only step requiring the strengthened hypothesis is ensuring that there is a path from $\vec{\cR}_1$ to $-\vec{\cR}_1$ along which the first reward function where some inequality becomes an equality is not trivial. With collinear visit counts, every pairwise value difference is a scalar multiple of a single inner product $\langle \cR, \cF^{\pi_1} - \cF^{\pi_0} \rangle$, so zeroing any one difference zeros them all. Non-collinearity guarantees three policies $\pi_a, \pi_b, \pi_c$ with $\cF^{\pi_b} - \cF^{\pi_a}$ and $\cF^{\pi_c} - \cF^{\pi_a}$ linearly independent. The set of rewards trivial on $\dPi$ must be orthogonal to both, giving a subspace of codimension at least $2$, which cannot disconnect $\mathbb{R}^{|\cS||\cA|}$. A path from $\vec{\cR}_1$ to $-\vec{\cR}_1$ therefore exists that avoids trivial rewards entirely, and the remainder of the argument proceeds unchanged.
\end{proof}

\section{Transition models for \texorpdfstring{\cref{fig:taxonomy,fig:gradients,fig:safe-horizon}}{Figures 1 to 3}}
\label{app:figure-mdp}

The MDP used in \cref{fig:taxonomy,fig:gradients,fig:safe-horizon} has $\cS = \{s_0, s_1, s_2\}$, $\cA = \{a_0, a_1\}$, $\cR(s,a) = \mathbbm{1}\{s = s_0\}$, $\gamma = 0.9$, and $d_0 = (1/3,\, 1/3,\, 1/3)$. Policies are parameterized by $\pi_\theta(a_0 \mid s) = \theta$ for all $s \in \cS$. Each transition model is specified by two stochastic matrices $P_{a_0}$ and $P_{a_1}$, where entry $(i,j)$ gives $\cT(s_j \mid s_i, a)$.

The shared transition model $\cT_1$ is
\[
P_{a_0}^{(1)} = \begin{pmatrix}
0.7 & 0.2 & 0.1 \\
0.5 & 0.3 & 0.2 \\
0.4 & 0.3 & 0.3
\end{pmatrix}, \qquad
P_{a_1}^{(1)} = \begin{pmatrix}
0.1 & 0.3 & 0.6 \\
0.1 & 0.2 & 0.7 \\
0.1 & 0.1 & 0.8
\end{pmatrix}.
\]
Under $\cT_1$, action $a_0$ concentrates transitions toward $s_0$ (the rewarding state), while action $a_1$ pushes toward $s_2$, making $J_1$ monotone increasing in $\theta$.

\paragraph{Panel (a): $\cT_2$ trivial.}
Both actions produce identical transitions, so $J_2$ is constant in~$\theta$:
\[
P_{a_0}^{(2a)} = P_{a_1}^{(2a)} = \begin{pmatrix}
0.4 & 0.3 & 0.3 \\
0.2 & 0.5 & 0.3 \\
0.3 & 0.3 & 0.4
\end{pmatrix}.
\]

\paragraph{Panel (b): $\cT_1$, $\cT_2$ equivalent.} We construct $\cT_2$ by blending $\cT_1$ toward the uniform kernel with mixing weight $\alpha = 0.3$, i.e.,
$\cT_2(s' \mid s, a) = 0.7 \cdot \cT_1(s' \mid s, a) + 0.3/|\cS|$:
\[
P_{a_0}^{(2b)} = \begin{pmatrix}
0.59 & 0.24 & 0.17 \\
0.45 & 0.31 & 0.24 \\
0.38 & 0.31 & 0.31
\end{pmatrix}, \qquad
P_{a_1}^{(2b)} = \begin{pmatrix}
0.17 & 0.31 & 0.52 \\
0.17 & 0.24 & 0.59 \\
0.17 & 0.17 & 0.66
\end{pmatrix}.
\]
This preserves the qualitative structure of $\cT_1$ (and thus the policy ordering) while attenuating the transition probabilities.

\paragraph{Panels (c) and (d): $\cT_1$, $\cT_2$ exploitable.} The transition models $\cT_2^{(c)}$ and $\cT_2^{(d)}$ were found by computational search over random transition matrices drawn from a Dirichlet distribution. Full numerical values and the search code are available in the supplementary material.

\section{Bound comparison}
\label{app:bound-comparison}

\begin{table}[h]
\centering
\begin{footnotesize}
\begin{tabular}{@{}llcccc@{}}
\toprule
Panel & Relationship & $\delta$ & Actual gap & \cref{thm:safe-horizon} bound (tight) & \cref{cor:sqrt-rule} bound (loose) \\
\midrule
(a) & Trivial       & 0.40 & 0    & 7.83 & 40 \\
(b) & Equivalent    & 0.14 & 0    & 5.58 & 14 \\
(c) & Exploitable   & 0.81 & 0.66 & 8.79 & 81 \\
(d) & Exploitable   & 0.64 & 4.35 & 8.52 & 64 \\
\bottomrule
\end{tabular}
\end{footnotesize}
\vspace{4mm}
\caption{Exploitation gaps for the transition model pairs in \cref{fig:taxonomy} at $\gamma = 0.9$. The actual gap is the largest $\varepsilon$ for which the pair is $\varepsilon$-exploitable. The \cref{thm:safe-horizon} and \cref{cor:sqrt-rule} bounds give the unique $\varepsilon$ above which each result guarantees $\varepsilon$-unexploitability. Both bounds are conservative.}
\label{tab:gaps}
\vspace{-4mm}
\end{table}

\Cref{tab:gaps} compares both bounds against the actual exploitation gaps for the transition model pairs in \cref{fig:taxonomy}. Both thresholds are conservative in every case and overestimate the actual gap. This conservatism is unavoidable in general: we prove \cref{thm:safe-horizon} is tight.


\newpage

\end{document}